\documentclass{miri-tech-article}

\bibliography{MIRIPublications.bib,General.bib,Inbox.bib}{}

\usepackage{amsmath,amssymb,amsthm,amsopn,amsfonts,latexsym}
\usepackage{breqn}
\usepackage[colorinlistoftodos]{todonotes}

\renewcommand{\U}{\util{U}}

\usepackage{tikz}
\usetikzlibrary{automata,positioning,arrows}
\usepackage{cyr}
\usetikzlibrary{arrows,matrix,positioning}
\usepackage{algorithm}
\usepackage[noend]{algpseudocode}

\makeatletter
\def\BState{\State\hskip-\ALG@thistlm}
\makeatother

\newtheorem{thm}{Theorem}[section]
\newtheorem{lem}[thm]{Lemma}

\newtheorem{cor}[thm]{Corollary}
\newtheorem{prop}[thm]{Proposition}

\newtheorem{question}[thm]{Question}

\newtheorem{Defn}[thm]{Definition}

\def\ep{\ve}

\def\k{\kappa}

\def\ve{\varepsilon}

\def\T{\mathbf{T}}

\def\<{\langle}
\def\>{\rangle}

\def\y{ {\text {\rm y}  } }

\def\Z{ {\text {\rm Z} } }

\def\U{{\text {\rm U} } }

\def\Q{{\text {\rm Q} } }

\def\0{{\mathbf 0}}

\def\I{\mathfrak{I}}

\newcommand{\e}{\ep}

\newcommand{\p}{\mathcal{P}}

\DeclareMathOperator*{\argmin}{arg\,min}

\title{Asymptotic Logical Uncertainty and The Benford Test}

\author{\normalsize{Scott Garrabrant\textsuperscript{1,2}, Siddharth Bhaskar\textsuperscript{1}, Abram Demski\textsuperscript{2,3},} \\ 
\normalsize{\bf{Joanna Garrabrant}, \bf{George Koleszarik}, \and \bf{Evan Lloyd}\textsuperscript{1}} \\
\textsuperscript{1}University of California, Los Angeles \\
\textsuperscript{2}Machine Intelligence Research Institute \\
\textsuperscript{3}University of Southern California
}

\begin{document}

\publishingnote{Technical Report 2015--11.}

\maketitle

\begin{abstract}
We give an algorithm $A_{L,T}$ which assigns probabilities to logical sentences. For any simple infinite sequence $\{\phi_{s_n}\}$ of sentences whose truth-values appear indistinguishable from a biased coin that outputs ``true'' with probability $p$, we have ${\lim_{n\rightarrow\infty} A_{L,T}(s_n)=p.}$
\end{abstract}

\section{Introduction}
Let $\phi_1,\phi_2,\ldots$ be a simple enumeration of all sentences in first order logic over ZFC. The goal of logical uncertainty is to construct an algorithm $M$ which on input $N$ outputs a probability~$M(N)$, which represents the probability that $\phi_N$ is true \cite{Christiano:2014a,Demski:2012,Gaifman:2004,Soares:2015g}.\footnote{The problem has also been studied in the case where we don't require computability even in the limit \cite{Gaifman:1964,Hutter:2013,Scott:1966}. The problem was first studied in the context of measures on Boolean algebras \cite{Horn:1948,Kelley:1959,Maharam:1947}.} This notion of probability does not refer to random variables. It refers to the degree of uncertainty that one might have about logical sentences whose truth-values have not been calculated.

Much work has been done on a related problem where $M$ on input $N$ outputs an infinite sequence of numbers and $M(N)$ is defined to be the limit of the sequence output by $M$ on input $N$ \cite{Christiano:2014a,Demski:2012,Soares:2015}. In this case, $M(N)$ is not computable, and can easily be $1$ for all provable $\phi$ and $0$ for all disprovable $\phi$, so all of the work is in figuring out how $M$ should behave when $\phi$ is independent of ZFC.

In this paper, we take a different approach, which we call \emph{asymptotic logical uncertainty}. We require that $M(N)$ be computable and have runtime bounded by some function of $N$. 

We propose as a baseline that any method of quickly assigning probabilities should be able to pass a test we call the \emph{Benford test}. Consider the infinite sequence of sentences $\{\phi_{s_n}\}$ given by $\phi_{s_n}=$ ``The first digit of $3\uparrow^n 3$ is a 1.'' We say that $M$ passes the Benford test if $$\lim_{n\rightarrow\infty}M(s_n)=\log_{10}(2)\approx.30103,$$
as prescribed by Benford's law. More generally, we say that $M$ passes the generalized Benford test if it converges to the correct probability on any similar infinite sequences whose truth values appear indistinguishable from independent flips of a biased coin.
We then give an algorithm $A_{L,T}$ which passes the generalized Benford test.

Logical uncertainty is one aspect of the problem of combining probability and logic, of which \emph{statistical relational learning} is another \cite{Getoor:2007}. Statistical relational learning addresses the problem of representing probabilistic models with logical structure, including regularities such as repeated entities and other complexities such as uncertainty about the number of entities. In contrast, logical uncertainty deals with uncertainty \emph{about} logic. As Paul Christiano put it: ``any realistic agent is necessarily uncertain not only about its environment or about the future, but also about the logically necessary consequences of its beliefs.'' \cite{Christiano:2014a}

\section{The Benford Test}
Benford's law states that in naturally occurring numbers, the leading digit $d\in\{1,\ldots,9\}$ of that number in base 10 occurs with probability $\log_{10}(1+\frac{1}{d})$. Many mathematical sequences have been shown to have frequencies of first digits that satisfy Benford's law \cite{Pietronero:2001}. In particular, the frequencies of the first digits of powers of $3$ provably satisfy Benford's law.

The function $3\uparrow^n k$ is defined by $3\uparrow^1 k=3^k$, $3\uparrow^n 1=3$, and ${3\uparrow^{n}k=3\uparrow^{n-1}(3\uparrow^{n}(k-1))}$. Throughout the paper, let $T(N)$ be an increasing time complexity function in the range of $N\leq T(N)\leq 3\uparrow^k N$ for some fixed $k$, and let $R(N)=T(N)N^4\log T(N)$. 

Consider the sequence $3\uparrow^n 3$. Clearly this sequence only contains powers of 3. We might hypothesize that the frequencies of the first digits in this sequence also satisfy Benford's law. However, $3\uparrow^n 3$ is very large, and first digit of $3\uparrow^n 3$ is probably very difficult to compute. It is unlikely that the first digit of $3\uparrow^3 3$ will ever be known.

If asked to quickly assign a probability to the sentence $\phi_{s_n}=$ ``The first digit of $3\uparrow^n 3$ is a 1,'' for some $n>2$, the only reasonable answer would be $\log_{10}(2)\approx.30103$. Note that $\phi_{s_n}$ is either true or false; there are no random variables. The probability here represents a reasonable guess in the absence of enough time or resources to compute $3\uparrow^n 3$.
\begin{Defn}\label{ben}
Let $M$ be a Turing machine which on input $N$ runs in time $O(R(N))$ and outputs a probability $M(N)$, which represents the probability assigned to $\phi_N$. We say that $M$ passes the Benford test if
$$\lim_{n\rightarrow\infty}M(s_n)=\log_{10}(2),$$
where $\phi_{s_n}=$ ``The first digit of $3\uparrow^n 3$ is a 1.''
\end{Defn}

It is easy to pass the Benford test by hard-coding in the probability. It is more difficult to pass the Benford test in a natural way. That the best probability to assign to $\phi_{s_n}$ is $\log_{10}(2)$ depends not only on the fact that the frequency with which $\phi_{s_n}$ is true tends toward $\log_{10}(2)$, but also on the fact that the sequence of truth-values of $\phi_{s_n}$ contains no patterns that can be used to quickly compute a better probability on some subsequence. We therefore assume that this sequence of truth-values is indistinguishable from a sequence produced by a coin that outputs ``true'' with probability $\log_{10}(2)$. Formally, we are assuming that ${S=\{s_n|n\in\mathbb{N}\}}$ is an {\em irreducible pattern} with probability $\log_{10}(2)$, as defined in the next section.

\section{Irreducible Patterns}\label{IPS}

Fix a universal Turing machine $U$ and an encoding scheme for machines, and let $U(M,x)$ denote running the machine $U$ to simulate $M$ with input $x$.

\begin{Defn}\footnote{We tailored this definition of irreducible pattern to our needs. The theory of algorithmic randomness may offer alternatives. However, algorithmic randomness generally considers all computable tests and focuses on the case where $p=\frac{1}{2}$ \cite{Ko:1986,Martin-Lof:1966,Downey:2010}. We believe that any reasonable definition inspired by algorithmic randomness would imply Definition \ref{IP}.} \label{IP}
Let $S\subseteq\mathbb{N}$ be an infinite subset of natural numbers such that $\phi_N$ is provable or disprovable for all $N\in S$, and there exists a Turing machine $Z$ such that $U(Z,N)$ runs in time $T(N)$ and accepts $N$ if and only if $N\in S$.

We say that $S$ is an \emph{irreducible pattern} with probability $p$ if there exists a constant $c$ such that for every positive integer $m\geq 3$ and every Turing machine $W$ expressible in $k(W)$ bits, if 
$$S'=\{N\in S\ |\ U(W,N)\text{ accepts in time }T(N)\}$$ has at least $m$ elements and $r(m,W)$ is the probability that $\phi_N$ is provable when $N$ is chosen uniformly at random from the first $m$ elements of $S'$, we have 
$$|r(m,W)-p|<\frac{c{k(W)}\sqrt{\log\log m}}{\sqrt{m}}.$$
\end{Defn}

\noindent The intuition behind the formula is that the observed frequency $r(m,W)$ for any sequence $S'$ we select should not stray far from $p$. The right hand side of the inequality needs to shrink slowly enough that a true random process would stay within it with probability 1 (given choice of $c$ sufficiently large to accommodate initial variation). The law of the iterated logarithm gives such a formula, which is also tight in the sense that we cannot replace it with a formula which diminishes more quickly as a function of $m$.

\begin{prop}
If we replace provability in Definition \ref{IP} with a random process, such that for each $N\in S$ the sentence $\phi_N$ is independently called ``provable'' with probability $p$, then $S$ would almost surely be an irreducible pattern with probability $p$.
\end{prop}
\begin{proof}
Fix a Turing machine $W$. By the law of the iterated logarithm, there exists a constant $c_1$ such that 
$$\limsup_{m\rightarrow\infty} \frac{|mr(m,W)-mp|}{\sqrt{m\log\log m}}=c_1$$
almost surely. Therefore $$\sup_{m} \frac{|mr(m,W)-mp|}{\sqrt{m\log\log m}}<\infty$$ almost surely. We will use $\Phi(W)$ as a shorthand for this supremum. For any $\varepsilon>0$, there therefore exists a $c_2$ such that $\mathbb{P}(\Phi(W)>c_2)\leq\varepsilon$.

We now show that $\mathbb{P}(\Phi(W)>2c_2+1)\leq\varepsilon^2$. By the chain rule for probabilities, it suffices to show that $\mathbb{P}((\Phi(W)>2c_2+1)|(\Phi(W)>c_2))\leq\varepsilon$. Assume $\Phi(W)>c_2,$ and Let $m_1$ be the first $m$ such that $$\frac{|mr(m,W)-mp|}{\sqrt{m\log\log m}}>c_2.$$ 

\noindent It suffices to show that the probability that there exists an $m_2$ with
$$\frac{|m_2r(m_2,W)-m_2p|}{\sqrt{m_2\log\log m_2}}-\frac{|m_1r(m_1,W)-m_1p|}{\sqrt{m_1\log\log m_1}} > c_2$$ is at most $\varepsilon$.

Observe that 

\begin{equation*}
\begin{split}
& \frac{|m_2r(m_2,W)-m_2p|}{\sqrt{m_2\log\log m_2}}-\frac{|m_1r(m_1,W)-m_1p|}{\sqrt{m_1\log\log m_1}}\\
& \leq\frac{|m_2r(m_2,W)-m_1r(m_1,W)-(m_2-m_1)p|}{\sqrt{(m_2-m_1)\log\log (m_2-m_1)}},
\end{split}
\end{equation*} 

and that the probability there exists an $m_2$ with
$$\frac{|m_2r(m_2,W)-m_1r(m_1,W)-(m_2-m_1)p|}{\sqrt{(m_2-m_1)\log\log (m_2-m_1)}}>c_2$$ is the same as the probability that $\Phi(W)>c_2$, which is at most $\varepsilon$.

We have thus shown that for every $\varepsilon,$ there exists a constant $c_3=c_2+1$ such that the probability that 
$\Phi(W)\geq 2^{\ell}c_3$ is at most $\varepsilon^{2^\ell}$.

Partition the set of all Turing machines into sets $\mathcal{W}_1,\mathcal{W}_2,\ldots,$ such that $\mathcal{W}_\ell$ contains all Turing machines expressed in at least $2^\ell$ but fewer than $2^{\ell+1}$ bits. The probability that a Turing machine $W$ in $\mathcal{W}_\ell$ violates \begin{equation*}|r(m,W)-p|<\frac{c_3{k(W)}\sqrt{\log\log m}}{\sqrt{m}},\tag{$\star$}\end{equation*} for any $m\geq 3$ is at most $\varepsilon^{2^\ell}$. The number of Turing machines in $\mathcal{W}_\ell$ is at most $2^{2^{\ell+1}}$, so the probability that there is any $W\in\mathcal{W}_\ell$ and $m\geq 3$ which violate 
($\star$) is at most $\varepsilon^{2^\ell}2^{2^{\ell+1}}$. Therefore, the probability that there is any Turing machine $W$ and $m\geq 3$ which violate
($\star$) is at most $$\sum_{\ell\in\mathbb{N}} \varepsilon^{2^\ell}2^{2^{\ell+1}}=\sum_{\ell\in\mathbb{N}} (4\varepsilon)^{2^\ell}.$$ For small enough $\varepsilon$ this goes to 0, so for large enough $c_3$, the probability that ($\star$) holds for all $W$ and $m$ goes to 1. Therefore, with probability 1, there exists a $c$ such that $$|r(m,W)-p|<\frac{c{k(W)}\sqrt{\log\log m}}{\sqrt{m}},$$ for all $m$ and $W$.
\end{proof}

We now use the concept of irreducible patterns to generalize the Benford test.

\begin{Defn} \label{genben}
Let $M$ be a Turing machine which on input $N$ runs in time $O(R(N))$ and outputs a probability $M(N)$, which represents the probability assigned to $\phi_N$. We say that $M$ passes the generalized Benford test if
$$\lim_{\substack{N\rightarrow\infty\\N\in S}}M(N)=p,$$
whenever $S$ is an irreducible pattern with probability $p$.
\end{Defn}

Note that if we conjecture that the $S$ from Definition \ref{ben} is an irreducible pattern with probability $\log_{10}(2)$, then any $M$ which passes the generalized Benford test also passes the Benford test.

\section{A Learning Algorithm}
We now introduce an algorithm $A_{L,T}$ that passes the generalized Benford test (see Algorithm 1).

Let $L$ be the Turing machine which accepts on input $N$ if ZFC proves $\phi_N$, rejects on input $N$ if ZFC disproves $\phi_N$, and otherwise does not halt. For convenience, in Algorithm 1, we define $\log q=1$ for $q<2$.

\begin{algorithm}\label{algorithm1}
\caption{$A_{L,T}(N)$}
\begin{algorithmic}[1]
\State $P=0$
\State $M=N$
\For{$j=0,\ldots, N$}
	\State $M_Y=0$ 
	\For{$Y$ a Turing machine expressible in $K_Y<\log N$ bits}
    	\State $M_X=N$
		\For{$X$ a Turing machine expressible in $K_X<\log N$ bits}
			\If{$U(X,N)$ and $U(Y,N)$ both accept in time $T(N)$}	
				\State $A=0$	
				\State $R=0$
				\State $i=1$
				\While {$i\leq N$}
					\If{$U(X,i)$ and $U(Y,i)$ both accept in time $T(i)$}
						\If{$U(L,i)$ accepts in time $T(N)$}
							\State $A=A+1$
						\ElsIf{$U(L,i)$ rejects in time $T(N)$}
							\State $R=R+1$
						\Else 
							\State $i=N$
						\EndIf
					\EndIf
					\State $i=i+1$		
				\EndWhile
				\State $F=A/(A+R)$
				\State $Q=A+R$
				\If{$\max\left({K_X},\frac{|F-\frac{j}{N}| \sqrt{Q}}{K_Y\sqrt{\log \log Q}}\right)<M_X$}
					\State $M_X=\max\left({K_X},\frac{|F-\frac{j}{N}| \sqrt{Q}}{K_Y\sqrt{\log \log Q}}\right)$
				\EndIf		
		\EndIf
	\EndFor
    \If{$M_X>M_Y$}		
    	\State $M_Y=M_X$
	\EndIf
    \EndFor
	\If{$M_Y<M$}
        \State $M=M_Y$
		\State $P=j/N$
	\EndIf
\EndFor
\State \Return $P$
\end{algorithmic}
\end{algorithm}

Let $TM(N)$ be the set of all Turing machines $X$ expressible in at most $\log N$ bits such that $U(X,N)$ accepts in time at most $T(N)$. The encoding of Turing machines must be prefix-free, which in particular means that no Turing machine is encoded in 0 bits. Let $J_N$ denote the set of rational numbers of the form $\frac{j}{N}$ with $j=0,\ldots,N$.

For $X$ and $Y$ Turing machines, let $K(X)$ be the number of bits necessary to encode $X$. Let $S'(X,Y)$ be the subset of natural numbers $i$ which are accepted by both $U(X,i)$ and $U(Y,i)$ in time at most $T(i)$. Let $Q_N(X,Y)$ be the greatest number less than or equal to $N$ such that for every $s$ in the first $Q_N(X,Y)$ elements of $S'$, $U(L,s)$ halts in time $T(N)$. Let $F_N(X,Y)$ be the proportion of the first $Q_N(X,Y)$ elements of $S'$ which $L$ accepts. Let 

\begin{equation*}
\begin{split}
& B_N(X,Y,P) \\
& =\max\left({K(X)},\frac{|F_N(X,Y)-P| \sqrt{Q_N(X,Y)}}{K(Y)\sqrt{\log \log Q_N(X,Y)}}\right).
\end{split}
\end{equation*}

\begin{lem}\label{out}
The output of $A_{L,T}$ on input $N$ is in
$$\argmin_{P\in J_N}\,\,\,\max_{Y\in TM(N)}\,\,\,\min_{X\in TM(N)}\,\,\,B_N(X,Y,P).$$
\end{lem}
\begin{proof}
The algorithm has three {\bf for} loops, the outer ranging over $j=0,\ldots N$ and the inner two ranging over $Y$ and $X$ respectively, both restricted to Turing machines expressible in $\log N$ bits. The condition on line 8 means that $X$ and $Y$ effectively range over all Turing machines in $TM(N)$, and $P=\frac{j}{N}$ ranges over $J_N$.

The inner {\bf while} loop will increment the variables $A$ or $R$ a total of exactly $Q_N(X,Y)$ times. Thus, $Q$ is set to $Q_N(X,Y)$ in line 22. Similarly, $F$ is sent to $F_N(X,Y)$ in line 21. Clearly $K_X$ and $K_Y$ are $K(X)$ and $K(Y)$ respectively. Therefore, the expression on lines 23 and 24 is $B_N(X,Y,P).$

Considering the {\bf for} loops from inner to outer, we minimize this quantity in $X$, maximize it in $Y$, and find $P$ of the form $j/N$ minimizing the whole quantity. The $P$ returned is therefore a minimizer of $$\max_{Y\in TM(N)}\,\,\,\min_{X\in TM(N)}\,\,\,B_N(X,Y,P).$$

\end{proof}

\noindent The code is not optimized for computational efficiency. The following proposition is just to ensure that the runtime is not far off from $T(N)$.

\begin{prop}
The runtime of $A_{L,T}(N)$ is in $O(R(N))=O(T(N)N^4\log T(N)))$.
\end{prop}
\begin{proof}
Simulating $U$ on any input for $T$ time steps can be done in time $cT\log T$ for some fixed constant $c$ \cite{Hennie:1966}. The bulk of the runtime comes from simulating Turing machines on lines 8, 13, 14, and 16. Each of these lines takes at most $cT(N)\log T(N)$ time, and we enter each of these lines at most $N^4$ times. Therefore, the program runs in time $O(T(N)N^4\log T(N))$.
\end{proof}

\section{Passing the Generalized Benford Test}
We are now ready to show that $A_{L,T}$ passes the generalized Benford test. The proof will use the following two lemmas.
\begin{lem}\label{XZ}
Let $S$ be an irreducible pattern with probability $p$, and let $Z$ be a Turing machine such that $U(Z,N)$ accepts in time $T(N)$ if and only if $N\in S$.

There exists a constant $C$ such that if $N\in S$, then there exists a $P\in J_N$ such that $$\max_{Y\in TM(N)}\,\,\,B_N(Z,Y,P)<C.$$
\end{lem}
\begin{proof}
Let $P=\frac{\lfloor pN\rfloor}{N}$. From the definition of irreducible pattern, we have that there exists $c$ such that for all $Y$,  
$$|F_N(Z,Y)-p|<\frac{c{K(Y)}\sqrt{\log\log Q_N(Z,Y)}}{\sqrt{Q_N(Z,Y)}}.$$ 
Clearly, 
\begin{equation*}
\begin{split}
&|P-p|\leq \frac{1}{N}\leq \frac{1}{Q_N(Z,Y)}\leq\frac{1}{\sqrt{Q_N(Z,Y)}}\\
&\leq \frac{{K(Z)K(Y)}\sqrt{\log\log Q_N(Z,Y)}}{\sqrt{Q_N(Z,Y)}}.
\end{split}
\end{equation*}

\noindent Setting $C=K(Z)+c$, we get 
\begin{equation*}
\begin{split}
|F_N(Z,Y)-P|&\leq|F_N(Z,Y)-p|+|P-p|\\
&<\frac{C{K(Y)}\sqrt{\log\log Q_N(Z,Y)}}{\sqrt{Q_N(Z,Y)}},
\end{split}
\end{equation*}

so 
$$\frac{|F_N(Z,Y)-P| \sqrt{Q_N(Z,Y)}}{K(Y)\sqrt{\log \log Q_N(Z,Y)}}<C.$$

\noindent Clearly, $K(Z)<C$, so $B_N(Z,Y,P)>C$ for all $Y$. Therefore, $$\max_{Y\in TM(N)}\,\,\,B_N(Z,Y,P)<C.$$
\end{proof}
\begin{lem}\label{YZ}
Let $S$ be an irreducible pattern with probability $p$, and let $Z$ be a Turing machine such that $U(Z,N)$ accepts in time $T(N)$ if and only if $N\in S$.

For all $C$, for all $\varepsilon>0$, for all $N$ sufficiently large, for all $P\in J_N$, if $N\in S$, and $$\min_{X\in TM(N)}\,\,\,B_N(X,Z,P)<C,$$ then $|P-p|<\varepsilon$.
\end{lem}
\begin{proof}
Fix a $C$ and a $\varepsilon>0$. It suffices to show that for all $N$ sufficiently large, if $N\in S$ and $|P-p|\geq\varepsilon$, then for all $X\in TM(N)$, we have $B_N(X,Z,P)\geq C.$

Observe that since $B_N(X,Z,P)\geq K(X)$, this claim trivially holds when $K(X)\geq C$. Therefore we only have to check the claim for the finitely many Turing machines expressible in fewer than $C$ bits. 

Fix an arbitrary $X$. Since $S$ is an irreducible pattern, there exists a $c$ such that 
$$|F_N(X,Z)-p|<\frac{c{K(Z)}\sqrt{\log\log Q_N(X,Z)}}{\sqrt{Q_N(X,Z)}}.$$
We may assume that $S'(X,Z)$ is infinite, since otherwise if we take $N\in S$ large enough, $X\notin TM(N)$. Thus, by taking $N$ sufficiently large, we can get $Q_N(X,Z)$ sufficiently large, and in particular satisfy $$\frac{\sqrt{Q_N(X,Z)}}{K(Z)\sqrt{\log \log Q_N(X,Z)}}\varepsilon\geq C+c.$$
Take $N\in S$ large enough that this holds for each $X\in TM(N)$ with $K(X)<C$, and assume $|P-p|\geq\varepsilon$. By the triangle inequality, we have 
\begin{equation*}
\begin{split}
&|F_N(X,Z)-P|\geq |P-p|-|F_N(X,Z)-p|\\
&\geq\varepsilon -\frac{c{K(Z)}\sqrt{\log\log Q_N(X,Z)}}{\sqrt{Q_N(X,Z)}}.
\end{split}
\end{equation*}
Therefore
\begin{equation*}
\begin{split}
&B_N(X,Z,P)\\
&\geq \frac{\left(\varepsilon -\frac{c{K(Z)}\sqrt{\log\log Q_N(X,Z)}}{\sqrt{Q_N(X,Z)}}\right) \sqrt{Q_N(X,Z)}}{K(Z)\sqrt{\log \log Q_N(X,Z)}}\\
&=\frac{\sqrt{Q_N(X,Z)}}{K(Z)\sqrt{\log \log Q_N(X,Z)}}\varepsilon-c\geq C,
\end{split}
\end{equation*}
which proves the claim.
\end{proof}
\begin{thm}\label{main}
$A_{L,T}$ passes the generalized Benford test.
\end{thm}
\begin{proof}
Let $S$ be an irreducible pattern with probability $p$. We must show that $$\lim_{\substack{N\rightarrow\infty\\N\in S}}A_{L,T}(N)=p.$$

\noindent Let $Z$ be a Turing machine such that $U(Z,N)$ accepts in time $T(N)$ if and only if $N\in S$. 

By considering the case when $X=Z,$ Lemma \ref{XZ} implies that there exists a constant $C$ such that for all $N$ sufficiently large, there exists a $P\in J_N$ such that $$\max_{Y\in TM(N)}\,\,\,\min_{X\in TM(N)}\,\,\,B_N(X,Y,P)<C.$$

\noindent Similarly, using this value of $C$, and considering the case where $Y=Z$, Lemma \ref{YZ} implies that for all $\varepsilon>0$, for all $N$ sufficiently large, for all $P\in J_N$ if $N\in S$, and $$\max_{Y\in TM(N)}\,\,\,\min_{X\in TM(N)}\,\,\,B_N(X,Y,P)<C,$$ then $|P-p|\leq\varepsilon$.

Combining these, we get that for all $\varepsilon>0$, for all $N$ sufficiently large, if $N\in S$ and if $P$ is in $$\argmin_{P\in J_N}\,\,\,\max_{Y\in TM(N)}\,\,\,\min_{X\in TM(N)}\,\,\,B_N(X,Y,P),$$ then $|P-p|\leq\varepsilon$.

Thus, by Lemma \ref{out}, we get that for all $\varepsilon>0$, for all $N$ sufficiently large, if $N\in S$, then $|A_{L,T}(N)-p|\leq\varepsilon,$ so $$\lim_{\substack{N\rightarrow\infty\\N\in S}}A_{L,T}(N)=p.$$
\end{proof}

\section{Final Remarks}
\begin{Defn}
Given a sentence $\psi$, consider the infinite sequence of integers $\{s^\psi_n\}$ given by $\phi_{s^\psi_0}=\psi$ and $\phi_{s^\psi_{n+1}}=\neg\neg \phi_{s^\psi_n}$. If a machine $M$ satisfies $$\lim_{n\rightarrow\infty} M(s^\psi_n)=p,$$ we say that $M$ converges to $p$ on $\psi$.
\end{Defn}
\begin{cor}\label{p1}
If $\psi$ is provable, then $A_{L,T}$ converges to 1 on $\psi$. If $\psi$ is disprovable, then $A_{L,T}$ converges to 0 on $\psi$.
\end{cor}
\begin{proof}
If $\psi$ is provable, then $\{s^\psi_n\}$ is an irreducible pattern with probably 1. If $\psi$ is disprovable, then $\{s^\psi_n\}$ is an irreducible pattern with probably 0. 
\end{proof}

If $\psi$ is neither provable nor disprovable, then it is not clear whether or not $A_{L,T}$ even converges on $\psi$.  

\begin{question}\label{converge?}
Does there exist a machine $M$ such that $M$ passes the generalized Benford test, and for each sentence $\psi$, there exists a $P(\psi)$ such that $M$ converges to $P(\psi)$ on $\psi$?
\end{question}
\begin{Defn}
A function $P$ from logical sentences to $[0,1]$ is called coherent if it satisfies the following three properties:
\begin{enumerate}
\item $P(\phi)=1$ for all provable $\phi$,
\item $P(\phi)=0$ for all disprovable $\phi$, and
\item $P(\phi)=P(\phi\wedge\psi)+P(\phi\wedge\neg\psi)$ for all $\phi$ and $\psi$.
\end{enumerate}\end{Defn}
\noindent Coherent functions correspond to probability distributions on the space of complete extensions of a given theory.

\begin{question}\label{coherent?}
Does there exist a machine $M$ and a coherent function $P$ such that $M$ passes the generalized Benford test, and for each sentence $\psi$, $M$ converges to $P(\psi)$ on $\psi$?
\end{question}

\printbibliography

\end{document}